\documentclass{ecai}
\usepackage{times}
\usepackage{graphicx}
\usepackage{latexsym}
\usepackage[numbers,sort&compress]{natbib}

\graphicspath{{./images/}}

\usepackage{subfigure}
\usepackage{amsmath}
\usepackage[utf8]{inputenc} 
\usepackage[T1]{fontenc}    
\usepackage{url}            
\usepackage{booktabs}       
\usepackage{amssymb}
\usepackage{amsfonts}       
\usepackage{nicefrac}       
\usepackage{microtype}      
\usepackage{color}
\usepackage{dsfont}

\usepackage{algorithm}
\usepackage{algorithmic}
\usepackage{flushend}

\newtheorem{lemma}{Lemma}
\newtheorem{corollary}{Corollary}
\newenvironment{proof}{\textit{Proof.} }{\hfill $\square$}

\renewcommand{\vec}[1]{\boldsymbol #1}
\newcommand{\mat}[1]{\mathbf{#1}}

\newcommand{\eps}{\varepsilon}

\newcommand{\A}{\mathcal{A}}
\newcommand{\D}{\mathcal{D}}
\renewcommand{\P}{\mat{P}}

\renewcommand{\S}{\mathcal{S}}

\newcommand{\IR}{\mathds{R}}

\newcommand{\PP}[2][]{\mathbb{P}_{#1}\left[#2\right]}
\newcommand{\EE}[2][]{\mathbb{E}_{#1}\left[#2\right]}

\newcommand{\set}[1]{\left\{#1\right\}}
\newcommand{\abs}[1]{\left|#1\right|}
\newcommand{\norm}[1]{\left\|#1\right\|}

\title{Class Teaching for Inverse Reinforcement Learners}
\author{Manuel Lopes\institute{INESC-ID, Instituto Superior Técnico, Portugal
email: manuel.lopes@tecnico.ulisboa.pt} \and Francisco S.\  Melo\institute{INESC-ID, Instituto Superior Técnico, Portugal
email: fmelo@inesc-id.pt} }

\begin{document}
\maketitle              

\begin{abstract}
In this paper we propose the first machine teaching algorithm for multiple inverse reinforcement learners. Specifically, our contributions are: (i) we formally introduce the problem of teaching a sequential task to a heterogeneous group of learners; (ii) we identify conditions under which it is possible to conduct such teaching using the same demonstration for all learners; and (iii) we propose and evaluate a simple algorithm that computes a demonstration(s) ensuring that all agents in a heterogeneous class learn a task description that is compatible with the target task. Our analysis shows that, contrary to other teaching problems, teaching a heterogeneous class with a single demonstration may not be possible as the differences between agents increase. We also showcase the advantages of our proposed machine teaching approach against several possible alternatives. 
\end{abstract}
%
%
\section{Introduction}
Machines can be used to improve education by providing personalized learning activities. Research on machine teaching and intelligent tutoring systems have considered various aspects of such machines~ \cite{patil2014,nkambou2010,Clement2015ITS,Davenport2012chemvlab,koedinger1997intelligent,Anderson1995Cognitive}. If we consider that a significant amount of teaching relies on providing examples, learning efficiency can be greatly improved if the teacher selects the examples that are more informative for each particular learner or class.

{\em Machine teaching} (MT) considers the problem of finding the smallest set of examples that allows a specific learner to acquire a given concept. MT sets itself apart from standard intelligent tutoring systems in that it explicitly considers a specific computational model of the learner \cite{Balbach09algteach,Zhu2013Machine,Zhu2015Machine,zhu2018overview}. The optimal amount of training examples needed to teach a target task to a specific learner is known as the {\em teaching dimension} (TD) of that task-learner pair~\cite{Goldman1995Complexity,Shinohara1991Teachability}. By optimizing the teaching dimension, machine teaching promises to strongly reduce the effort required from both learner and teacher.

Machine teaching, much like intelligent tutoring systems, can be applied in several real world problems. We are motivated by examples where we need to teach task that are sequential in nature: cognitive tasks e.g. algebra, or algorithms; motor tasks e.g. industrial maintenance or assembly. In many such cases, we need to train a large number of learners, who might have different cognitive and motor skills.

%
%
%

Most MT research so far has focused on single-learner settings in non-sequential tasks---such as Bayesian estimation and classification~\cite{Goldman1995Complexity,Shinohara1991Teachability,Balbach09algteach,Zhu2013Machine,Zhu2015Machine,zhu2018overview}. Recently, however, some works have considered the extension of the machine teaching paradigm to novel settings. For example:
\begin{enumerate}
\item Some works have investigated the impact of group settings on machine teaching results. Zhu et al. \cite{zhu2017no} show that, by dividing a group of learners in small groups, it is possible to attain a smaller teaching dimension. The work of Yeo et al. \cite{yeo2019iterative} generalize those results for more complex learning problems, and consider additional differences between the learners, e.g. learning rates.
\item Some works \cite{Haug2018TeachingIR,Walsh_dynamicteaching,Melo18ijcai} investigate the impact that the mismatch between the learner and the teacher's model of the learner may have in the teaching dimension---a situation that is particularly relevant in group settings. The aforementioned works focus on supervised learning settings, although some more recent works have started to explore inverse reinforcement learning settings \cite{kamalaruban2019interactive}.
\item Other works have considered machine teaching in sequential decision tasks. Cakmak and Lopes \cite{cakmak2012algorithmic} introduce the first machine teaching algorithm for sequential decision tasks (i.e., when the learners are inverse reinforcement learners). Brown and Niekum \cite{brown2018machine} propose an improved algorithm that takes into consideration reward equivalence in terms of the target task representation. The work of Rafferty et al. \cite{rafferty2015inferring} considers sequential tasks in a different way, instead of evaluating the quality of learning based on the match between the demonstrated and the learned policy, it infers the understanding of the task by estimating the world model that the learners inferred.
\end{enumerate}

In this paper, we build on the extensions above and consider the problem of teaching a sequential task to a group of learners (a ``class''). We henceforth refer to a setting where a single teacher interacts with multiple (possibly different) learners as {\em class teaching}. We follow Cakmak and Lopes \cite{cakmak2012algorithmic} in assuming that the learners are inverse reinforcement learners \cite{ng2000algorithms}, and address the problem of selecting a demonstration that ensures that all learners are able to recover a task description that is ``compatible'' with the target task, in a sense soon to be made precise. 

Teaching a sequential task in a class setting, however, poses several additional complications found neither in single-agent settings \cite{Haug2018TeachingIR,Walsh_dynamicteaching,cakmak2012algorithmic,brown2018machine}, nor on estimation/classification settings \cite{zhu2017no,Walsh_dynamicteaching,yeo2019iterative}.

In this setting we need to teach not only one particular learner, but a whole diverse group of learners. The teacher needs to guarantee that all learners learn, while delivering the same lecture to everyone. Learner diversity might have different origins, they may go from different learning rates or prior information, to having a completely learning algorithm. Intuitively speaking we may think that if the differences are large then each learner needs a particular demonstration and class teaching is not possible. Nevertheless, quantifying what are large differences is not trivial. For example, in the family of tasks considered in Zhu et al. \cite{zhu2017no,yeo2019iterative} learners have large differences in their prior information. But, no matter the amount of differences, all learners can be taught with the same demonstration, even if a larger number of samples is required. We want to understand what happens in sequential tasks, and quantify which differences between learners may still allow to teach all of them simultaneously, and which differences can not be addressed.

In this work we discuss the challenges arising when extending machine teaching of sequential tasks in class settings. We contribute the first formalization of the problem from the teacher's perspective. We then contribute an analysis of the problem, identifying conditions under which it is possible to teach a heterogeneous class with a common demonstration. From our analysis, we then propose the first class teaching algorithm for sequential tasks and illustrate its advantages against other more ``naive'' alternatives.



\section{Background}
\label{sec:RL}

In this section we go over key background concepts upon which our work builds, both to set the nomenclature and the notation. We go over Markov decision problems (MDPs) \cite{putterman2005markov}, inverse reinforcement learning \cite{ng2000algorithms} and machine teaching in RL settings \cite{cakmak2012algorithmic,brown2018machine}. 

\subsection{Markov Decision Problems}
\label{sec:ReinforcementLearning}

A Markov decision problem (MDP) is a tuple $(\S,\A,\P,r,\gamma)$, where $\S$ is the state space, $\A$ is the action space, $\P$ encodes the transition probabilities, where 
\begin{equation*}
\P(s'\mid s,a)=\PP{S_{t+1}=s'\mid S_t=s,A_t=a},
\end{equation*}
and $S_t$ and $A_t$ denote, respectively, the state and action at time step $t$. The function $r:\S\to\IR$ is the reward function, where $r(s)$ is the reward received by the agent upon arriving at a state $s\in\S$. Finally, $\gamma\in[0,1)$ is a discount factor. 

A {\em policy} is a mapping $\pi:\S\to\Delta(\A)$, where $\Delta(\A)$ is the set of probability distributions over $\A$. Solving an MDP amounts to computing the optimal policy $\pi^*$ that maximizes
\begin{equation*}
V^\pi(s)\triangleq\EE{\sum_{t=0}^\infty\gamma^tr(s)\mid S_0=s,A_t\sim\pi(\cdot\mid S_t)}
\end{equation*}
for all $s\in\S$. In other words, we have that $V^{\pi^*}(s)\geq V^\pi(s)$ for all policies $\pi$ and states $s$. We henceforth denote by $\pi^*(r)$ the optimal policy with respect to the MDP $(\S,\A,\P,r,\gamma)$, where $\S$, $\A$, $\P$, and $\gamma$ are usually implicit from the context. Writing the value function $V^\pi$ as a vector $\vec{v}^\pi$, we get
\begin{equation}\label{Eq:Vector-form}
\vec{v}^\pi
  =\vec{r}+\gamma\P_\pi\vec{v}^\pi
  =(\mathbf{I}-\gamma\P_\pi)^{-1}\vec{r},
\end{equation}
where $\P_\pi$ is a matrix with component $ss'$ given by:
\begin{displaymath}
[\vec{P_\pi}]_{ss'}=\sum_{a\in\A}\pi(a\mid s)\P(s'\mid s,a).
\end{displaymath}


\subsection{Inverse Reinforcement Learning}
\label{sec:IRL}

In inverse reinforcement learning (IRL),  we are provided with a ``rewardless MDP'' $(\S,\A,\P,\gamma)$ and a sample of the policy $\pi$, or a trajectory, and wish to determine a reward function $r^*$ such that $\pi$ is optimal with respect to $r^*$, i.e., $\pi=\pi^*(r^*)$ for the resulting MDP. If $\pi$ is optimal then, given an arbitrary policy $\pi'$, 
\begin{equation*}
\vec{r}+\gamma\P_\pi\vec{v}^\pi\succeq\vec{r}+\gamma\P_{\pi'}\vec{v}^\pi,
\end{equation*}
where we write $\succeq$ to denote element-wise inequality. Using \eqref{Eq:Vector-form}, the solution must verify the constraint
\begin{equation}\label{Eq:IRL-Inequality}
(\P_\pi-\P_{\pi'})(\mat{I}-\gamma\P_\pi)^{-1}\vec{r}\succeq\vec{0}.
\end{equation}

Unfortunately, the constraint in \eqref{Eq:IRL-Inequality} is insufficient to identify $r^*$. For one, \eqref{Eq:IRL-Inequality} is trivially verified for $\vec{r}=\vec{0}$. More generally, given a policy $\pi$, there are multiple reward functions that yield $\pi$ as the optimal policy. In the context of an IRL problem, we say that two reward functions $r$ and $r'$ are {\em policy equivalent} if $\pi^*(r)=\pi^*(r')$.%
\footnote{This happens, for example, if $r-r'$ is a {\em potential function} \cite{ng99icml}.} 
Moreover, the computation of the constraint in \eqref{Eq:IRL-Inequality} requires the learner to access the complete policy $\pi$. In practice, however, it is inconvenient to explicitly enunciate $\pi$. Instead, the learner is provided with a {\em demonstration} consisting of a set
\begin{equation*}
\D=\set{(s_n,a_n),n=1,\ldots,N}
\end{equation*}
where, if $(s,a)\in\D$, $a$ is assumed optimal in state $s$.

To address the two difficulties above, it is common to treat \eqref{Eq:IRL-Inequality} as a constraint that the target reward function must verify, but select the latter so as to meet some additional regularization criterion $J$, in an attempt to avoid the trivial solution \cite{ng2000algorithms}. For the purpose of this work, we re-formulate IRL as
\begin{equation}\label{Eq:IRL-Optimization-1}
\begin{aligned}
  \max         &&&\vec{1}^\top\vec{v}\\
  \text{s.t.} &&&(\vec{p}(s_n,a_n)-\vec{p}(s_n,b))\vec{v}\succeq\eps,  \forall (s_n,a_n)\in\D, b\in\A\\
           &&& 0\preceq\vec{v}\preceq\frac{R_{max}}{1-\gamma},
\end{aligned}
\end{equation}
where $\vec{p}(s,a)$ is the row-vector with element $s'$ given by $\P(s'\mid s,a)$. In \eqref{Eq:IRL-Optimization-1} we directly solve for $V^\pi$ instead of $r^*$, and then compute $r^*$ as 
\begin{equation*}
\vec{r}^*=\vec{v}-\gamma\max_{a\in\A}\P_a\vec{v}.
\end{equation*}
The IRL formulation in \eqref{Eq:IRL-Optimization-1} implicitly assumes a reward $r\leq R_{max}$, which has no impact of the representative power of the solution. Moreover, it deals with the inherent ambiguity of IRL by maximizing the value of all states while imposing that the ``optimal actions'' are at least $\eps$ better than sub-optimal actions. The proposed formulation, while closely related to the simpler approaches in \cite{ng2000algorithms}, is simpler to solve and less restrictive in terms of assumptions.

We emphasize that previous works on machine teaching in sequential tasks assume that IRL learners turn a demonstration into constraints that the reward function must verify, like those in \eqref{Eq:IRL-Inequality}. However, such constraints are built in a way that requires the learner to know (or, at least, be able to sample) the teacher's policy $\pi$ \cite{cakmak2012algorithmic,brown2018machine}. As argued before, this is often inconvenient/unrealistic. Our formulation in \eqref{Eq:IRL-Optimization-1} is an original contribution that readily overcomes such limitation and has interest on its own. 

In the remainder of the paper, we refer to an ``IRL agent'' as defined by a rewardless MDP $(\S,\A,\P,\gamma)$ that, given a demonstration $\D$, outputs a reward $r(\D)$ obtained by solving \eqref{Eq:IRL-Optimization-1}.


\subsection{Machine Teaching in IRL}
\label{section:teachingIRL}

\begin{figure}[!tb]
\centering
\subfigure[IRL agent $A$.]{\label{Fig:Learner-1}
  \includegraphics[width=0.45\columnwidth]{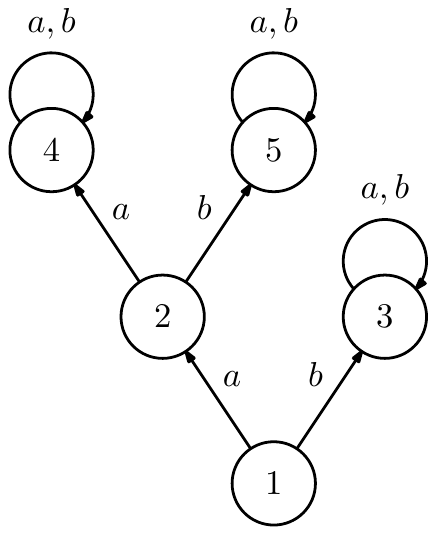}  
} \subfigure[IRL agent $B$.]{\label{Fig:Learner-2}
  \includegraphics[width=0.45\columnwidth]{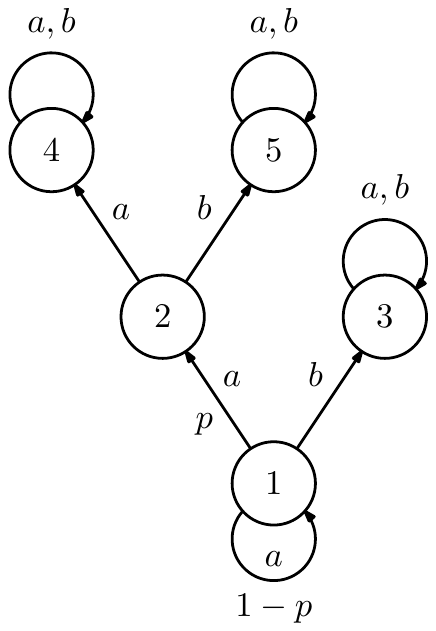}  
}
\caption{Diagram representing two IRL agents (all unmarked transition are deterministic). The two agents are similar in all states except $1$, where action $a$ always succeeds for agent $A$ but only succeeds with probability $p$ for agent $B$.}
\label{Fig:Two-IRL-learners}
\end{figure}

We now consider the problem of {\em teaching} an IRL agent. In particular, given an IRL agent described by a rewardless MDP $(\S,\A,\P,\gamma)$ and a target reward function $r^*$, we want to determine the ``most concise'' demonstration $\D$ such that $r(\D)$ is policy-equivalent to $r^*$, i.e.,
\begin{equation*}
\pi^*(r^*)=\pi^*(r(\D)).    
\end{equation*} 
By ``most concise'' we imply that there is a function, $\mathsf{effort}$, that measures the teaching effort associated with any demonstration $\D$ (for instance, the number of examples in $\D$). Teaching an IRL agent can thus be formulated as solving
\begin{equation}\label{Eq:MT-IRL-Optimization}
\begin{aligned}
  \min_\D  &&& \mathsf{effort}(\D)\\
  \text{s.t.} &&& \pi^*(r^*)=\pi^*(r(\D)).
\end{aligned}
\end{equation}
Consider, for example, the IRL agent $A$ in Fig.~\ref{Fig:Learner-1}, defined as the rewardless $(\set{1,2,3,4,5},\set{a,b},\P,\gamma)$, where the edges represent the transitions associated with the different actions (unmarked edges correspond to deterministic transitions) and $\gamma>0.5$. If the target reward is
\begin{equation}\label{Eq:Reward}
\vec{r}^*=
\begin{bmatrix}  
0 & 0 & 1 & 0 & 2  
\end{bmatrix}^\top,
\end{equation}
the optimal value function is given by
\begin{equation*}
\vec{v}^*=\frac{1}{1-\gamma}
\begin{bmatrix}  
2\gamma^2 &
2\gamma &
1 &
0 & 
2
\end{bmatrix}^\top,
\end{equation*}
and the optimal policy selects action $a$ in state $1$ and action $b$ in state $2$, since $2\gamma>1$. Since both actions are equal in the remaining states, the most succinct demonstration should be, in this case,
\begin{equation*}
\D=\set{(1, a),(2, b)}.
\end{equation*}

As another example, consider the IRL agent $B$ in Fig.~\ref{Fig:Learner-2}. This learner is, in all aspects, similar to IRL agent $A$ except that action $a$ is now stochastic in state $1$ and succeeds only with probability $p$. The optimal value function is now
\begin{equation*}
\vec{v}^*=\frac{1}{1-\gamma}
\begin{bmatrix}  
u &
2\gamma &
1 &
0 & 
2
\end{bmatrix},
\end{equation*}
where
\begin{equation*}
u=\max\set{\frac{2\gamma^2p}{1-\gamma(1-p)},1}.
\end{equation*}
Then, if 
\begin{equation*}
    p>\frac{1-\gamma}{\gamma(2\gamma-1)},
\end{equation*}
the optimal policy is the same as in the previous case, as is the most concise demonstration. If, instead, the reverse inequality holds, the optimal policy is now to select action $b$ in both states $1$ and $2$, and the best demonstration is
\begin{equation*}
    \D=\set{(1, b),(2, b)}.
\end{equation*} 
Finally, if 
$p=\frac{1-\gamma}{\gamma(2\gamma-1)},$
then both actions are equally good in state $1$, and the most concise demonstration is just $\D=\set{(2, b)}$.


\section{Class Teaching of Sequential Tasks}

In this section we present our main contributions. We start by formalizing the problem of class teaching for IRL learners, i.e. teach simultaneously multiple IRL learners. We then identify necessary conditions that ensure that we can teach all learners in a class simultaneously, i.e. using the same demonstrations for all. We finally provide a first algorithm that is able to teach under these conditions.


\subsection{Teaching a Class of IRL Learners}%
\label{Subsec:Teaching-IRL-class}

We now consider a teacher facing a heterogeneous class of $L$ IRL learning agents, each one described as a rewardless MDP $M_\ell=(\S,\A,\P_\ell,\gamma)$. Note that we allow different learners to have different models.%
\footnote{For sake of clarity most of the discussion considers only differences in terms of transition probabilities. In the results we also show difference in terms of different discount $\gamma$. Differences in features can also be considered with minor changes.}
We assume that the teacher perfectly knows the models $M_1,\ldots,M_L$ and that the learners all adopt the IRL formulation in \eqref{Eq:IRL-Optimization-1}, given a demonstration $\D$ consisting of a set of state-action pairs.

Given a target reward function $r^*$, the goal of the teacher is, once again, to find the ``most concise'' demonstration $\D$ that ensures that $r_\ell(\D)$ is compatible with $r^*$, where $r_\ell(\D)$ is the reward computed by the IRL agent $\ell$ upon observing $\D$. In other words, the goal of the teacher is to solve the optimization problem
\begin{equation}\label{Eq:MT-MultiIRL-Optimization}
\begin{aligned}
  \min_\D  &&& \mathsf{effort}(\D)\\
  \text{s.t.} &&& \pi_\ell^*(r^*)=\pi_\ell^*(r_\ell(\D)), \text{for $\ell=1,\ldots,L$.}
\end{aligned}
\end{equation}
For the sake of concreteness, we henceforth consider $\mathsf{effort}(\D)=\abs{\D}/\abs{\S}$, roughly corresponding to the ``percentage'' of demonstrated states. The constraint in \eqref{Eq:MT-MultiIRL-Optimization} states that the teacher should consider only demonstrations $\D$ that ensure the optimal policies for $(\S,\A,\P_\ell,r^*,\gamma)$ and $(\S,\A,\P_\ell,r_\ell(\D),\gamma)$ to be the same, for all $\ell=1,\ldots,L$. 

In general, the problem \eqref{Eq:MT-MultiIRL-Optimization} may not have a solution. In fact, there may be no single demonstration that ensures that all learners recover a reward function compatible with $r^*$. Consider for instance a class comprising agents $A$ and $B$ from Fig.~\ref{Fig:Two-IRL-learners}. Suppose that the target reward is that in \eqref{Eq:Reward} and that 
\begin{equation*}
p<\frac{1-\gamma}{\gamma(2\gamma-1)}.
\end{equation*}
If we provide the demonstration $\D=\set{(2,b)}$, the only constraint imposed by such demonstration is that $V(5)\geq V(4)+\eps$, which leads to the solution
\begin{equation*}
\vec{v}=\frac{1}{1-\gamma}
\begin{bmatrix}  
1 & 1 & 1 & 1-\eps(1-\gamma) & 1
\end{bmatrix}^\top,
\end{equation*}
corresponding to the reward
\begin{equation}\label{Eq:Reward-A1}
\vec{r}=
\begin{bmatrix}  
1 & 1 & 1 & 1-\eps(1-\gamma) & 1
\end{bmatrix}^\top.
\end{equation}
Such reward does not verify the constraint in \eqref{Eq:MT-MultiIRL-Optimization}. For example, the policy that selects $a$ and $b$ in state $1$ with equal probability is optimal with respect to the reward in \eqref{Eq:Reward-A1}, both for $A$ and $B$. However, it is not optimal with respect to $r^*$ for neither of the two. Repeating the derivations above for the demonstration $\D=\set{(1,a),(2,b)}$, we immediately see that the reward $r_A(\D)$ will verify the constraint in \eqref{Eq:MT-MultiIRL-Optimization} but not the reward $r_B(\D)$. Conversely, if $\D=\set{(1,b),(2,b)}$, we immediately see that  $r_B(\D)$ will verify the constraint in \eqref{Eq:MT-MultiIRL-Optimization} but not $r_A(\D)$.

The example above can be distilled in the following result, where a demonstration $\D$ is {\em complete} if there is a pair $(s,a)\in\D$ for every $x\in\S$. 

\begin{lemma}\label{Lemma:Complete-demos}
For two complete demonstrations $\D_1$ and $\D_2$ and two arbitrary IRL agents $A$ and $B$ described, respectively, by the rewardless MDPs $(\S,\A,\P_A,\gamma)$ and $(\S,\A,\P_B,\gamma)$, then 
\begin{equation*}
  \pi^*_A(r_A(\D_1))=\pi^*_B(r_B(\D_2))
\end{equation*}
if and only if $\D_1=\D_2$.
\end{lemma}
\begin{proof}
By definition, a complete demonstration includes a pair for every state $s\in\S$ with a corresponding optimal action. The constraints implied by the demonstration will necessarily lead both agents to learn similar policies. Conversely, if the agents learn different policies, either the demonstrations are different or incomplete.
\end{proof}

As argued before, assuming that the policy is provided to the learners in full (i.e., the demonstration is complete) is often unrealistic. In the more natural situation of an incomplete demonstration, the conclusion of Lemma~\ref{Lemma:Complete-demos} no longer holds, as established by the following result.




\begin{theorem}\label{Theo:sDdP}
Let $\S$ and $\A$ denote arbitrary finite state and action spaces, with $\abs{\S}>1$ and $\abs{\A}>1$, and $\D\subset\S\times\A$ an incomplete demonstration. Then, there exist two IRL agents $(\S,\A,\P_A,\gamma)$ and $(\S,\A,\P_B,\gamma)$ such that
\begin{equation}\label{Eq:Different-policies}
\pi^*_A(r_A(\D))\neq\pi^*_B(r_B(\D)).
\end{equation}
In other words, an incomplete demonstration may lead to different policies in different agents.
\end{theorem}
\begin{proof}
The proof proceeds by explicitly building two IRL agents, by replicating the structure of the agents in Fig.~\ref{Fig:Two-IRL-learners}. For simplicity of exposition, we assume that no state in $\S$ appears in more than one pair in $\D$, although the proof also holds in the converse case with due modifications.

Let $\D$ be an incomplete demonstration. This means that there is at least one state $s_0\in\S$ that does not appear in any pair in $\D$. Since, by assumption, $\abs{\S}>1$, let $s_1$ be some state in $\S$ such that $s_1\neq s_0$. Moreover, let $a_0,a_1$ denote two arbitrary actions in $\A$, with $a_0\neq a_1$ (recall that $\abs{\A}>1$, by assumption). We now construct the transition probabilities for agents $A$ and $B$ such that \eqref{Eq:Different-policies} holds.

For every $(s_n,a_n)\in\D$, we let 
\begin{multline*}
\P_A(s'\mid s_n,a)=\P_B(s'\mid s_n,a)\\
=\begin{cases}
1 & \text{if $s'=s_0$ and $a=a_n$;}\\
1 & \text{if $s'=s_1$ and $a\neq a_n$;}\\
0 & \text{otherwise}.
\end{cases}
\end{multline*}
In other words, according to both $\P_A$ and $\P_B$, the optimal actions always lead to $s_0$, and the sub-optimal actions always lead to $s_1$. Additionally, we consider that all states other than $s_0$ that do not appear in any pair in $\D$ transition are absorbing according both to $\P_A$ and $\P_B$. It follows that $\D$ implies a single constraint in the optimization problem \eqref{Eq:IRL-Optimization-1}, namely that $V(s_0)\geq V(s_1)+\eps$. Setting
\begin{align*}
\P_A(s'\mid s_0,a)&=\begin{cases}
1 & \text{if $s'=s_0$ and $a=a_0$;}\\
1 & \text{if $s'=s_1$ and $a\neq a_0$;}\\
0 & \text{otherwise;}
\end{cases}\\
\P_B(s'\mid s_0,a)&=\begin{cases}
1 & \text{if $s'=s_0$ and $a=a_1$;}\\
1 & \text{if $s'=s_1$ and $a\neq a_1$;}\\
0 & \text{otherwise,}
\end{cases}
\end{align*}
it follows that $\pi_A(s_0;r_A(\D))=a_0$ and $\pi_B(s_0;r_B(\D))=a_1$, and the proof is complete.
\end{proof}

This is a negative result for class teaching: we show that the differences between the agents may imply that the same reward leads to different optimal policies which, in turn, implies that there are cases where the same demonstration will lead to rewards that are not ``compatible'' with the target policy (i.e., do not verify the constraint in \eqref{Eq:MT-MultiIRL-Optimization}). This is particularly true for classes where the learners exhibit large differences among themselves. We now identify necessary conditions to ensure that two learners recover reward functions compatible with $r^*$ from a common demonstration $\D$.

\begin{lemma}
\label{Lemma:Polddemod}
Given two MDPs $(\S,\A,\P_A,r^*,\gamma)$ and $(\S,\A,\P_B,r^*,\gamma)$, if $\pi_1^*(r^*)\neq\pi^*_2(r^*)$, then the two IRL agents $(\S,\A,\P_A,\gamma)$ and $(\S,\A,\P_B,\gamma)$ require different demonstrations $\D_A$ and $\D_B$ in order to recover a reward compatible with $r^*$.
\end{lemma}
\begin{proof}
Let $s_0\in\S$ be such that $\pi_A^*(s_0;r^*)=a_0$ and $\pi_B^*(s_0;r^*)=a_1$, with $a_0\neq a_1$, and suppose that we provide a common demonstration to agents $A$ and $B$. Clearly, if either $(s_0,a_0)$ or $(s_0,a_1)$ appear in $\D$, one of the agents will learn a reward that is not compatible with $r^*$. On the other hand, if $s_0$ does not appear in $\D$, both agents will learn rewards according to which the policy that selects $a_1$ and $a_2$ with equal (and positive) probability is optimal, which are incompatible with $r^*$.
\end{proof}

Lemma~\ref{Lemma:Polddemod} establishes that, in general, we cannot expect to achieve successful class teaching, where the same examples can be used by everyone. It also provides a verified way to test how different the learner can be before we need personalized teaching. We get the following corollary.

\begin{corollary}[Possibility of Class Teaching]\label{Corollary}
 Given two IRL agents $(\S,\A,\P_A,\gamma)$ and $(\S,\A,\P_B,\gamma)$, it is possible to class-teach a reward $r^*$ if and only the optimal policies for the MDPs $(\S,\A,\P_A,r^*,\gamma)$ and $(\S,\A,\P_B,r^*,\gamma)$ are the same, i.e., if $\pi_A^*(r^*)=\pi^*_B(r^*)$. 
\end{corollary}

Corollary~\ref{Corollary} states the main challenge of class teaching in sequential tasks: if the differences between learners imply different optimal policies, they cannot be taught with a common demonstration.

\subsection{Proposed Algorithm}
\label{sec:Proposed Algorithm}

In Section~\ref{Subsec:Teaching-IRL-class} we showed that class teaching is not possible in the general case. Our results also provide criteria to determine whether, in a particular situation, class teaching is possible or not. 

From the teacher's perspective, the goal is both to teach the correct task and to reduce the effort in teaching.%
\footnote{Recall that we consider the effort to depend directly on the number of demonstrations provided.} 
When providing a demonstration to a class, the effort is the same, independently of the number of learners in the class. So if an example is required for multiple learners, it is more efficient to provide it for the class as a whole than individually. Conversely, when demonstrations are contradictory, the teacher should provide a different demonstration to each learner individually (with the corresponding increase in the teaching effort). We note that some examples may be required for a learner but redundant for another. In such situation, and taking into account the way we measure effort, the teacher can still provide such examples to the class without added effort or the danger of preventing correct learning. 

\begin{algorithm}[!htbp]
\caption{Teaching Multiple IRL Learners}
\label{tab:OptTeachX0}
\begin{algorithmic}
\REQUIRE IRL learners $A$ and $B$
\REQUIRE Set of possible initial states $\S_0$
\REQUIRE Target reward $r^*$
\STATE Compute set of demonstrations starting in $s_0\in\S_0$
\STATE Check optimal policies of each learner
\STATE Select demonstrations consistent with the optimal policies of learners
\STATE Provide class demonstrations
\FOR{$\ell=A,B$}
\STATE Provide non-redundant demonstrations to learner $\ell$
\ENDFOR
\end{algorithmic}
\end{algorithm}

From all these considerations, we propose the simple approach in Alg.~\ref{tab:OptTeachX0}, where the teacher builds the demonstration as a set of optimal trajectories of state-action pairs generated from some initial state $s_0$. This extends the algorithm in \cite{brown2018machine} for the class setting. We restrict the role of the teacher to that of selecting the initial states. Then,  the algorithm proceeds as follows: it identifies the optimal policy for each learner given the target reward $r^*$. The teacher then demonstrates to the class those trajectories that are compatible across learners, and to each learner individually those trajectories that are specific to that learner's optimal policy. \footnote{Without lack of generality the algorithm is presented for 2 learners. For more than 2 learners we can select demonstrations than are informative to any subset of learners to reduce the effort.}

\paragraph{Complexity} We can analyze the complexity of this work along several dimensions. Complexity of verifying if class teaching is possible or not implies comparing the optimal policies for the different learners. This comparison requires solving the MDP for each learner, which has a polynomial complexity.

On the other hand, computing which demonstrations to provide to each learner is linear in the number of learners and possible initial states. However, if we want to reduce the teaching effort by providing the most efficient demonstrations, we must identify which demonstrations introduce redundant constraints. This can be done through linear programming \cite{brown2018machine}, and requires solving as many linear programs as the size of the initial demonstrations set. In this case, since linear programming is solvable in polynomial time, we again obtain polynomial complexity.


\paragraph{Approximate Solutions}
In our discussion so far, we consider only exact demonstrations and investigate conditions under which all learners in the class are able to recover the desired reward function (or a policy equivalent one) exactly. We could, however, consider situations where some error is acceptable. 

\subparagraph{Error in the policy}
For instance, we can consider an extended setting that allows small errors in the reward recovered by (some of) the learners. In such setting, we could define an $\epsilon>0$ such as $|\pi^*(r^*)-\pi^*(r(D))|<\epsilon$, for example by combining our approach with that in \cite{Haug2018TeachingIR}. Such approximate setting could allow reductions to the teaching effort in the case where class teaching is possible. 

However, the impossibility results we established still hold even in the approximate case. In fact, when class teaching is not possible, it is possible to find $\epsilon_L>0$ such that $|\pi^*(r^*)-\pi^*(r(D))|\geq\epsilon_L$, i.e., we cannot reduce the error arbitrarily (for otherwise class teaching would be possible). The example of Fig.~\ref{Fig:Two-IRL-learners} shows one such case, where the same demonstration, if provided to the two learners, would result in an error that could not be made arbitrarily small.

\subparagraph{Loss in value}

Another setting is to consider that we allow the learners to learn different rewards as long as the expected cumulative discounted reward is not far from the optimal. Let us consider a scenario with two IRL agents, $A$ and $B$, each one described as a rewardless MDP $(\S,\A,\P_\ell,\gamma)$, $\ell=A,B$. Further assume that $\pi_A^*(r^*)\neq\pi_B^*(r^*)$, and suppose that we provide both learners with a complete demonstration $\D$ such that
\begin{equation*}
\pi_A^*(r_A(\D))=\pi_A^*(r^*).
\end{equation*} 
In other words, learner $A$ is able to recover from $\D$ a reward that is policy equivalent to $r^*$, i.e., such that
\begin{equation*}
V^{\pi_A^*(r_A(\D))}(s)=V^*(s)
\end{equation*} 
for all $s\in\S$. Lemma~\ref{Lemma:Polddemod} ensures that learner $B$ will recover a reward $r_B(\D)$ such that 
\begin{equation*}
\pi_B^*(r_B(\D))=\pi_A^*(r_A(\D))\neq\pi_B^*(r^*).
\end{equation*} 

In spite of our impossibility results, we can nevertheless provide an upper bound to how much the performance of learner $B$ strays from that of learner $A$. For simplicity of notation, we henceforth write $V^{\pi_\ell}$ to denote $V^{\pi_\ell^*(r_\ell(\D))}$, for $\ell=A,B$. Then,
\begin{align*}
\vec{v}^{\pi_A}-\vec{v}^{\pi_B}
  &=\vec{r}^*_{\pi_A}+\gamma\P_{A,\pi_A}\vec{v}^{\pi_A}-\vec{r}^*_{\pi_B}-\gamma\P_{B,\pi_B}\vec{v}^{\pi_B}\\
    &=\gamma(\P_{A,\pi_A}\vec{v}^{\pi_A}-\P_{B,\pi_B}\vec{v}^{\pi_B}),
\end{align*}
where the last equality follows from the fact that $\vec{r}^*_{\pi_A}=\vec{r}^*_{\pi_B}$, since $\pi_B^*(r_B(\D))=\pi_A^*(r_A(\D))$. Some manipulation yields
\begin{multline*}
\vec{v}^{\pi_A}-\vec{v}^{\pi_B}
  =\frac{\gamma}{2}\Big[(\P_{A,\pi_A}+\P_{B,\pi_B})(\vec{v}^{\pi_A}-\vec{v}^{\pi_B})\\
  +(\P_{A,\pi_A}-\P_{B,\pi_B})(\vec{v}^{\pi_A}+\vec{v}^{\pi_B})\Big].
\end{multline*}
Defining
\begin{align*}
    \bar{\P}_\pi&=\frac{1}{2}(\P_{A,\pi_A}+\P_{B,\pi_B}) &
    \bar{v}_\pi&=\frac{1}{2}(\vec{v}^{\pi_A}+\vec{v}^{\pi_B}),
\end{align*}
we get
\begin{equation*}
\vec{v}^{\pi_A}-\vec{v}^{\pi_B}
  =\gamma\left[\mat{I}-\gamma\bar{\P}_\pi\right]^{-1}(\P_{A,\pi_A}-\P_{B,\pi_B})\bar{\vec{v}}_\pi.
\end{equation*}
Noting that $\bar{\P}_\pi$ is still a stochastic matrix, the inverse above is well defined. Computing the norm on both sides, we finally get, after some shuffling,
\begin{equation*}
\norm{\vec{v}^{\pi_A}-\vec{v}^{\pi_B}}_2
  \leq\frac{\gamma}{1-\gamma}\norm{\P_{A,\pi_A}-\P_{B,\pi_B}}_2\norm{\bar{\vec{v}}_\pi}_2.
\end{equation*}
As expected, the difference in performance between agents $A$ and $B$ grows with the difference between the corresponding transition probabilities. Thus we cannot always reduce the error arbitrarily, and so cannot do class teaching, but we can bound the error.

\paragraph{Differences between learners}
Up to now we presented the differences between learners in terms of having different world models. This is in practice having different $\P$. There are other differences that also exist in many cases. People can have different discount factors $\gamma$. This can represent situations where the survival rate is different or situations where the agent expected more volatility in the environment. This situation is similar to the previous one. Different $\gamma$ values can give rise to different policies and so class teaching is only possible if they are the same.
Another possible difference is the different agents having different representations for the reward. For instance, the reward can be written as a linear combination of features, $R=w\phi(s)$. This case does not create any problem, assuming that both spaces of features allow to learn the correct reward, they have the same optimal policy and what happens is that they learn different vectors of the linear combination $w$. If the space of features is not rich enough to represent the optimal reward function, then we are in the approximate setting (\cite{kamalaruban2019interactive}) where class teaching might not always be possible.

\section{Examples}

In this section we provide practical examples of when class teaching can, or cannot, be made in different scenarios. We present two simple scenarios motivated by potential applications in human teaching, and two extra scenarios that show other possibilities of our algorithm, namely that it works in random MDPs, and that they also accept differences in terms of different $\gamma$.

\paragraph{Scenario~1. Brushing teeth (Cognitive training)}

Training sequential tasks is very important for many real world applications. For instance elderly whose cognitive skills are diminishing often struggle to plan simple tasks such as brushing their teeth or dressing up~\cite{si2007context}. Motivated by such situations, we model how to train a group of people on the steps to brush their teeth (Fig.~\ref{Fig:Brushing}). To brush the teeth, the brush ($B$) and toothpaste ($P$) must be picked; the brush must be filled ($F$) with toothpaste; only then brushing will lead to clean teeth ($C$). People may forget to put the paste, or may have coordination problems and be unable to hold the brush while placing the paste. 

\begin{figure}[!tb]
\centering
\subfigure[Teeth brushing example.]{\label{Fig:Brushing}
  \includegraphics[width=\columnwidth]{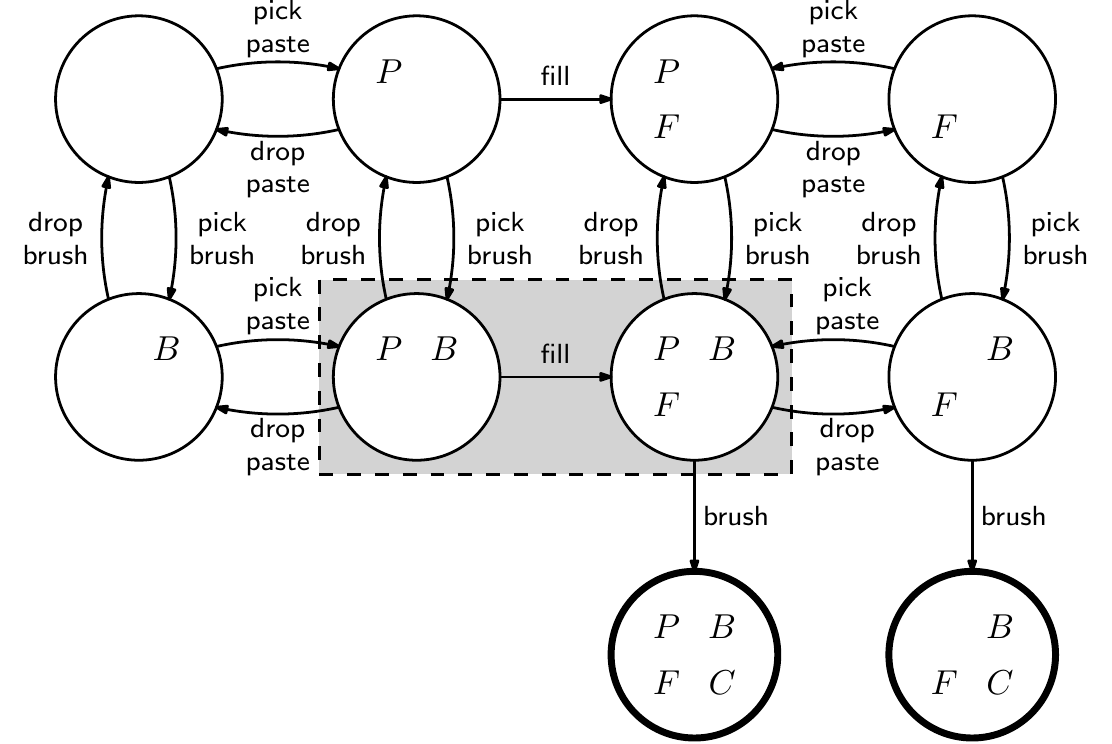}
} \hfill 
\subfigure[Addition with carry.]{\label{Fig:Addition}
  \includegraphics[width=\columnwidth]{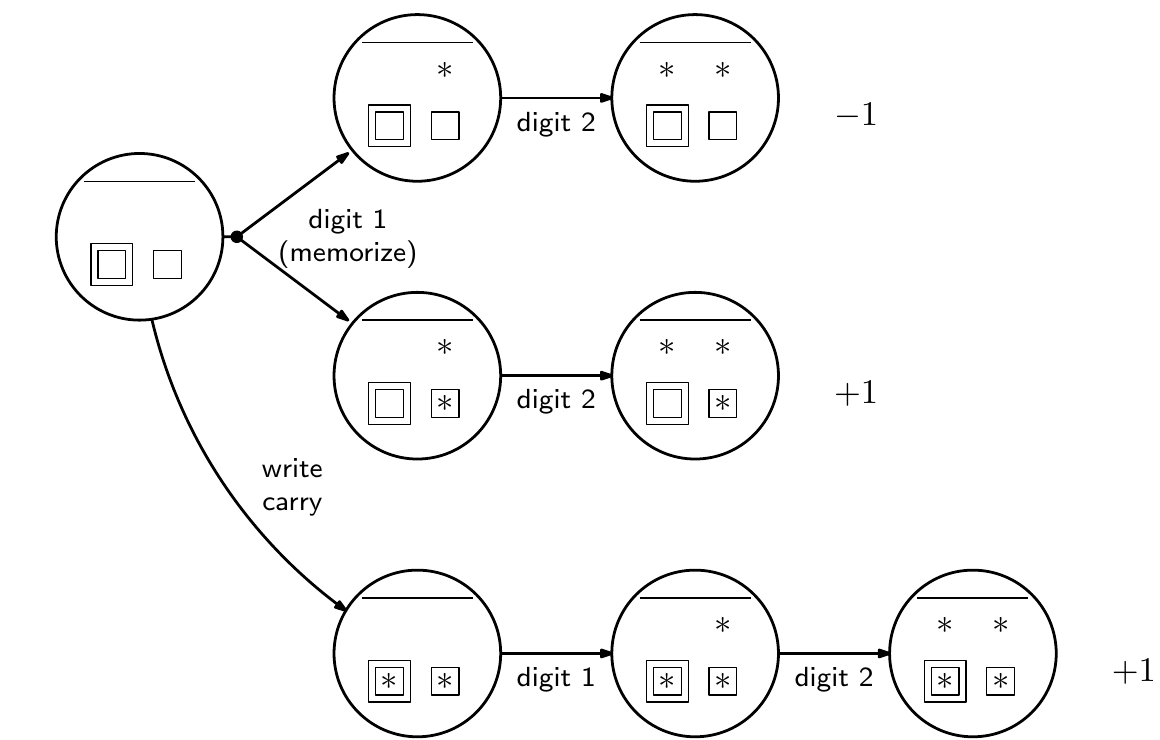}
}
\caption{\subref{Fig:Brushing} Brushing teeth MDP. Each state is described by 4 binary features: $P$ (holding paste), $B$ (holding brush), $F$ (brush with paste), $C$ and (teeth clean). We consider the case where one user that cannot hold two objects at the same time and so some of the states are inaccessible (states in the shaded region). \subref{Fig:Addition} Simple model of 2-digit addition with a single carry. The asterisks represent digits. Some learners are able to memorize the carry digit (the single square) and do not have to explicit write it (the double square). If memorizing may fail (top path), the teacher might suggest to write the carry digit explicitly (lower path). Without writing the carry digit, a learner with difficulties may forget it (the top branch has higher probability), obtaining the wrong result.} 
\label{fig:Tasks}
\end{figure}

\paragraph{Scenario~2: Addition with Carry (Education)}

When teaching mathematical operations, teachers need to choose among different algorithms to perform those operations, taking into account the level of the learners, their capabilities for mental operations and how much practice they had~\cite{putnam1987structuring}. Let us consider addition with carry. For some learners, it might be useful to write down the carry digit to avoid confusion. A more advanced learner might find it confusing or even boring to be forced to make such auxiliary step. We can model this problem as the MDP in Fig.\ref{Fig:Addition}. The asterisks indicate which of the digits of the result have been computed (top). The square indicates whether or not the carry digit is memorized, while the double square indicates whether or not the carry digit is written down. A learner with bad memory may prefer to write down the carry digit, for otherwise there is a larger probability of forgetting it and getting a wrong result.

\paragraph{Scenario~3: Random MDP}

To further illustrate the application of our approach in a more abstract scenario (ensuring that our algorithm is not exploiting any particular structure of the previous scenarios), we also consider a randomly generated MDP with multiple states (5-20 states), actions (3-5 actions) and rewards. The transition probabilities and reward are sampled from a uniformed distribution.

\paragraph{Scenario~4: Difference in discount factor $\gamma$}

To consider other types of differences we now consider that the two learners are described by the MDP as in Fig.~\ref{Fig:Two-IRL-learners}a) but $\gamma_A=0.9$ and $\gamma_B= 0.01$ respectively. In this case one learner wants rewards sooner than the other. In this case the policy in state 0 is different and so class teaching in not possible.

\subsection*{Results}

\begin{table*}[!htbp]
\centering
\caption{Results for class teaching in 4 different MDPs. We present the average effort and the average relative loss for individual teaching, teaching consider the model of a single agent, or both (our approach). As baselines we use: i) Class i where we provide both learner the optimal demonstration for learner i (expected to have worse quality with minimum effort); ii) Individual where we provide each learner with a different demonstration (expected to have best quality with extra effort). Our algorithm is able to teach the task with less effort and with the same quality.}
\label{tab:results}
\begin{tabular}{lrrrrrrrr}
\toprule
            & \multicolumn{2}{c}{\bf 1.~Brushing}   
            & \multicolumn{2}{c}{\bf 2.~Addition} 
            & \multicolumn{2}{c}{\bf 3. Random MDP}
            & \multicolumn{2}{c}{\bf 4. Different $\gamma$}
            \\\midrule
            & $\frac{V-V^*}{V^*}$ & \textsf{effort} & $\frac{V-V^*}{V^*}$ & \textsf{effort} & $\frac{V-V^*}{V^*}$ & \textsf{effort} & $\frac{V-V^*}{V^*}$ & \textsf{effort}
            \\\midrule
Class $A$  & $-0.04$ & $0.5$ & $-0.05$ & $0.375$ & $-0.044$ & $0.83$ & $-0.49$ & $0.4$  \\
Class $B$  & $-0.5$  & $0.4$ & $-0.44$ & $0.375$ & $-0.012$ & $0.83$ & $-0.2$  & $0.22$ \\
Individual & $0.0$   & $0.9$ & $\mathbf{0.0}$ & $\mathbf{0.625}$ & $0.0$ & $1.67$  & $\mathbf{0.0}$ & $\mathbf{0.6}$\\
Algorithm~\ref{tab:OptTeachX0} & $\mathbf{0.0}$ & $\mathbf{0.8}$ & $\mathbf{0.0}$ & $\mathbf{0.625}$ & $\mathbf{0.0}$ & $\mathbf{1.17}$ & $\mathbf{0.0}$ & $\mathbf{0.6}$ \\\bottomrule
\end{tabular}
\end{table*}

We present the results on the four described scenarios in Tab.~\ref{tab:results}. As far as we know, our work is the first to address machine teaching to multiple IRL learners. For this reason, we compare our algorithm with two natural baselines: teaching each agent individually, where each agent gets a ``personalized'' demonstration (denoted as ``Individual''); and teaching the whole class considering the model of a single agent $\ell$ (denoted as ``Class $\ell$''. We expect individual teaching to provide the best results in terms of estimating the correct reward, but at a cost in terms of effort. Class teaching, on the other hand, will reduce the effort but might not allow correct learning. Our algorithm was able to always teach correctly while never having more effort than individual teaching.

In the Brushing scenario we can see that neither Class teaching could teach the task even if the effort was lower. The individual could teach everything but with a higher effort. Our algorithm could reduce the effort while still guaranteeing teaching. In the arithmetic case the results are similar. As the policies go through completely different paths it is not possible to reduce effort. For random MDPs we can see that the qualitative results are still the same with higher reductions in effort. Scenario 4 also shows that we can consider other differences among learners.

\section{Conclusions}

In this work we formalized the problem of class teaching for IRL learners, studied its properties, introduced an algorithm to address this problem and provided some simulations to illustrate the theoretical results. We identify a set of conditions to verify if class teaching is possible or not. Contrary to several recent results for density estimation and supervised learning (\cite{zhu2017no,yeo2019iterative}) where we can always do classroom teaching (with an extra effort), in the case of inverse reinforcement learning it is not always possible. We provided a computational way to verify if class teaching is possible or not.

We illustrated the theoretical results in four different tasks and confirmed that the class teaching approach is able to teach as well as individualized teaching with the additional advantage of a lower effort. The results provided in this work provide a quantitative evaluation of when class teaching is possible. As a side contribution, we showed also a simpler way to solve the IRL problem using directly the value function. 

On way to avoid the constraints we have derived from the differences between the learners is to allow some noise in the learning process. This approach has been considered in other works, for supervised learning settings, and it is interesting to consider it in this setting when it is detected that class teaching is not possible. We would then need to verify if it is more efficient to teach using a class but accept some errors, or increase the effort and do a partition of the class and avoid errors. Even allowing some errors there are situations where there is no solution. We showed even if we allow a small error in the reward, there will always be cases where the error in the policies will be too large. Another variant we considered is to accept an error in the reward if the value obtained with the learned policy under the true reward is not far from the optimal. Again we showed that it is not always possible to reduce the error arbitrarily but we can have a bound on that error. 

We can imagine several applications of this work in teaching humans. For those kinds of applications the complexity of the algorithm is not a problem but the assumption of knowing the learner's decision-making process (i.e., MDP) is. We have to consider how to include interaction in the teaching process to overcome such problems as was done for other teaching problems, e.g. \cite{Melo18ijcai}.

Other applications of machine teaching include the study of possible attacks of machine learners, e.g. \cite{mei2015mtattacks}. We can use our approach to see if a set of learners can be attacked simultaneously or not.


\section*{Acknowledgments}
 This work was supported by national funds through the Portuguese Funda\c{c}\~{a}o para a Ci\^{e}encia e a Tecnologia.


\bibliographystyle{ecai}
\bibliography{ECML19}

\end{document}